\pdfoutput=1
\documentclass[conference]{IEEEtran}
\usepackage{amsmath,amssymb,amsthm}
\usepackage{booktabs}
\usepackage{graphicx}
\usepackage{cite}

\newtheorem{theorem}{Theorem}
\newtheorem{proposition}{Proposition}
\newtheorem{lemma}{Lemma}
\newtheorem{corollary}{Corollary}
\theoremstyle{definition}
\newtheorem{definition}{Definition}
\newtheorem{assumption}{Assumption}
\theoremstyle{remark}
\newtheorem{remark}{Remark}

\newcommand{\E}{\mathbb{E}}
\newcommand{\KL}{D_{\mathrm{KL}}}
\newcommand{\cD}{\mathcal{D}}
\newcommand{\cN}{\mathcal{N}}
\newcommand{\bth}{\theta}
\newcommand{\etaG}{\eta^{G}}

\title{Exact Fusion and Coordinated Exploration in Multi-Robot Active Inference}
\author{\IEEEauthorblockN{Peng Wu\IEEEauthorrefmark{1},
Mohsen Imani\IEEEauthorrefmark{2},
Amidu Kamara\IEEEauthorrefmark{3},\\
Md Tamzeed Islam\IEEEauthorrefmark{4},
Seyede Fatemeh Ghoreishi\IEEEauthorrefmark{1},
and Mahdi Imani\IEEEauthorrefmark{1}}\\[3pt]
\IEEEauthorblockA{\IEEEauthorrefmark{1}Northeastern University, Boston, MA, USA.
\{p.wu, f.ghoreishi, m.imani\}@northeastern.edu\\
\IEEEauthorrefmark{2}University of California, Irvine, CA, USA. m.imani@uci.edu\\
\IEEEauthorrefmark{3}U.S. Department of Homeland Security, Washington, DC, USA. amidu.kamara@hq.dhs.gov\\
\IEEEauthorrefmark{4}Oracle Corporation, USA. tamzeed.islam@oracle.com}}

\begin{document}
\maketitle
\thispagestyle{empty}\pagestyle{empty}

\begin{abstract}
Robot teams that learn a common environment model exchange belief summaries and plan by the expected information gain of their actions. Under conjugate exponential-family beliefs the shared belief is counted once per robot at two points: at fusion, the product of local posteriors counts the common prior $n$ times, and at planning, every robot scores its plan under the same belief and the team converges on the same unknown. Both errors are removed by adding evidence increments to the shared natural parameter, realized increments at fusion and expected increments at planning. The expected increment of a committed teammate gives the next robot its conditional gain; corrected gains sum to the joint gain, the redundancy removed equals the total correlation of the planned observation streams, and sequential commitment keeps the $1/2$ greedy guarantee. The expected increment is exact for Gaussian beliefs with fixed sampling paths and for Dirichlet beliefs under the novelty approximation of discrete active inference, whose team objective has a closed concave form within an explicit bound of the exact mutual information, and fails for finite hypothesis classes, where a short exact enumeration replaces it. Experiments on cooperative RockSample, foraging, and field monitoring show that fusion correction leaves exploration redundancy unchanged, anticipated evidence removes it, and sequential commitment recovers most of the value of centralized joint planning at cost linear in the team size.
\end{abstract}

\section{Introduction}\label{sec:intro}
Robot teams operating in partially known environments accumulate local experience of the same world, and the team performs best when this experience is pooled \cite{burgard05,singh07}. Belief sharing \cite{wang25,friston24} exchanges summaries of a common environment model rather than raw observations, so that a receiver needs neither the sender's data nor its sensor and motion models. This paper studies belief sharing in teams whose robots plan by the value of information, as active inference agents do, and shows that the shared belief causes the same error twice, when local beliefs are fused and when plans are chosen.

The first point is synchronization. Every local posterior contains the previously shared belief as a factor, so a product of local posteriors counts that factor once per robot, the double-counting problem of decentralized estimation \cite{grime94,dagan23}, which \cite{wang25} corrects for finite hypothesis classes by dividing out $n-1$ copies of the common belief. In a conjugate exponential family the division becomes an addition: the shared natural parameter is kept once and each robot adds the evidence increment of its new data, which recovers the centralized posterior exactly and remains inside the family for any subset of received increments.

The second point is planning, and it has not been treated for active inference agents that share beliefs. The novelty term of the expected free energy is the expected information gain about the model parameters. When every robot evaluates it under the same shared belief, every robot identifies the same most informative region and credits its own plan with the full gain of resolving it, so the team obtains that gain once and leaves the remaining uncertainty untouched. We show that the overcount equals the total correlation of the planned observation streams, that it is fixed by the plans and the shared belief, and that no fusion rule applied after the plans have been executed can reduce it.

The planning correction applies the same additive update to evidence not yet collected: under conjugacy a committed robot can compute and broadcast the expected increment of its plan, and later robots plan under the shared parameter with these increments added. One message format therefore serves both purposes, its realized form updating the shared model and its expected form coordinating the plans, and a receiver applies either without the sender's sensor or motion models. The contributions are the following.
\begin{enumerate}
\item One additive update of the shared natural parameter expresses exact fusion with realized evidence increments and coordinated planning with expected increments (Secs.~\ref{sec:retro} and \ref{sec:pro}); the corrected gains sum to the joint gain of the team, and sequential commitment retains the $1/2$ guarantee of greedy submodular maximization.
\item Theorem~\ref{thm:exact} gives a sufficient condition for the expected increment to yield the exact conditional gain: the gain depends on the increment only through a component that the plan fixes. It holds for Gaussian beliefs with deterministic sampling paths and for Dirichlet beliefs under the novelty approximation of discrete active inference, where the team objective has the closed concave form \eqref{eq:dirteam} within the bound of Lemma~\ref{lem:mn} of the exact mutual information, and fails for finite hypothesis classes, where a short exact enumeration replaces it.
\item Experiments in the three families confirm that fusion correction leaves exploration redundancy unchanged, that anticipated evidence removes it, and that sequential commitment recovers most of the value of centralized joint planning at linear cost.
\end{enumerate}

\section{Related Work}\label{sec:related}
\emph{Decentralized Bayesian fusion.} The removal of common information before local estimates are combined goes back to the channel filter \cite{grime94} and the Bayesian committee machine \cite{tresp00}; heterogeneous decentralized data fusion \cite{dagan23} and the analysis of shared priors \cite{wu24} treat the general case, and \cite{wang25} divides the product of local posteriors by $n-1$ copies of the shared belief for finite hypothesis classes. In information form the same correction is an addition of increments \cite{grime94,grocholsky03}, of which the site-factor exchange of partitioned variational inference \cite{ashman22} is the approximate counterpart. The Q-learning policies federated in \cite{wang25} do not represent the value of information, so the planning error studied here does not arise there; concurrent reinforcement learning identifies redundant exploration as a failure of naive posterior sampling and randomizes the per-agent beliefs \cite{dimakopoulou18}, whereas the correction proposed here is deterministic and carries an approximation guarantee.

\emph{Multi-agent active inference.} Federated inference and belief sharing \cite{friston24} lets agents broadcast posteriors over states so that the team converges on a common percept, and naive posterior sharing produces echo chambers \cite{catal24}; both aim at agreement about the current state, and neither divides exploration among agents or fuses model parameters across synchronization rounds. In robotics, active inference has driven single-robot exploration \cite{wakayama25}, and the theory-of-mind planner of \cite{tom25}, closest to our planning correction, coordinates teams by recursive reasoning over a joint policy tree at a cost exponential in the number of agents; Sec.~\ref{sec:exp-tom} compares against a centralized enumeration over the same policy space, which upper-bounds it.

\emph{Multi-robot informative planning.} Coordinated exploration discounts frontier utilities that other robots will reach \cite{burgard05}. Sequential greedy maximization of conditional mutual information carries a $1/2$ guarantee under matroid constraints \cite{fisher78,krause05,singh07,corah18} and underlies decentralized planners in which each robot plans against the predicted information matrices of its teammates \cite{grocholsky03,atanasov15,schlotfeldt18,kantaros21}, distributed sequential greedy assignment \cite{corah17}, and Dec-MCTS \cite{best19}; batch Bayesian optimization hallucinates observations at the posterior mean, which updates the variance and leaves the mean unchanged \cite{desautels14,ginsbourger10}. Our Gaussian instantiation coincides with these devices; new here are the Dirichlet instantiation, the sufficiency characterization of Thm.~\ref{thm:exact} with its negative result for finite hypothesis classes, the closed-form submodular objective \eqref{eq:dirteam}, and one additive message for both fusion and coordination.

\section{Problem Formulation}\label{sec:formulation}
A team of $n$ robots operates in a common environment with unknown parameter $\bth \in \Theta$. Robot $i$ has state $s^i_t$, actions $a^i_t$, observations $o^i_t$, and task preferences $C^i(o)$. The joint generative model is \begin{equation}\label{eq:model} P(\bth)\prod_{i=1}^{n}\prod_{t} P(o^i_t \mid s^i_t,\bth)\, P(s^i_t \mid s^i_{t-1}, a^i_{t-1}, \bth), \end{equation} and $\cD^i_{k}$ denotes the states, actions, and observations recorded by robot $i$ up to time $k$.

\begin{assumption}[Factorized likelihood]\label{ass:ci}
No robot's action enters another robot's transition or observation model, and each robot selects actions as a function of data only. Consequently $P(\cD^{1:n} \mid \bth) = c \prod_i P(\cD^i \mid \bth)$, where $P(\cD^i \mid \bth)$ collects the factors of \eqref{eq:model} that involve robot $i$ and $c$ the action probabilities, which do not depend on $\bth$; for fixed policies the observation streams of different robots are conditionally independent given $\bth$.
\end{assumption}

\begin{assumption}[Synchronization]\label{ass:sync}
The robots synchronize every $T$ steps, and at the first synchronization they hold a common prior $q^G_0(\bth)$; Proposition~\ref{prop:proper} covers increments that arrive late or not at all.
\end{assumption}

\begin{assumption}[Conjugacy]\label{ass:conj}
Beliefs over $\bth$ belong to a conjugate exponential family $q_\eta(\bth) = \exp(\langle\eta, u(\bth)\rangle - \Phi(\eta))$ with natural parameter $\eta \in \mathcal{N}$. Conditioning on data $\cD$ maps $\eta \mapsto \eta + \Delta\eta(\cD)$, where the evidence increment $\Delta\eta(\cD)$ is a sum of per-observation terms.
\end{assumption}

Assumption~\ref{ass:ci} remains valid when policies are coupled through shared beliefs, since action selection depends on $\bth$ only through the data. The exactness statements of Secs.~\ref{sec:retro} and \ref{sec:pro} are made for observed robot states, as in the experiments, for which $\Delta\eta$ is the sufficient statistic of the local data; when the robot state is latent, $\Delta\eta$ is the expected sufficient statistic under the local state posterior \cite{dacosta20}, and the same operations return the variational approximation of discrete active inference in place of the exact posterior. Throughout, $H$ denotes entropy and $I(X;Y) = H(Y) - H(Y \mid X)$ mutual information under the robots' predictive beliefs; conditional mutual information satisfies the chain rule $I(X; Y_{1:n}) = \sum_i I(X; Y_i \mid Y_{1:i-1})$.

Between synchronizations each robot runs discrete active inference \cite{dacosta20}: it accumulates its evidence increment $\Delta\eta^i := \Delta\eta(\cD^i_{\mathrm{new}})$ and selects a policy from a candidate set $\Pi^i$ by minimizing the expected free energy \begin{equation}\label{eq:efe} G^i(\pi) = -\E[C^i(o^i_\pi)] - I(s^i_\pi; o^i_\pi) - I(\bth;\, o^i_\pi), \end{equation} whose terms are the expected task value, the expected information gain about the robot's own state, and the expected information gain about the shared parameter, $I(\bth; o^i_\pi) = \E_{Q(o\mid\pi)}[\KL(q(\bth \mid o) \,\|\, q(\bth))]$. The last term, called novelty in active inference, is what informative path planning maximizes \cite{krause05,singh07}, so the analysis applies to any planner that scores plans by expected information gain. The problem is to specify what the robots exchange so that (i) the shared belief equals the centralized posterior $P(\bth \mid \cD^{1:n})$ after every synchronization and (ii) each robot evaluates its plan by the information gain conditional on the plans of the robots that committed before it, with a guarantee on the joint gain of the round, while no robot needs the sensor or motion models of another.

\section{Exact Fusion by Evidence Increments}\label{sec:retro}
Suppose that at synchronization $k$ all robots hold the common belief $q_{\etaG_k}$, and over the next $T$ steps robot $i$ updates its local copy to $\eta^i = \etaG_k + \Delta\eta^i$. Naive fusion forms the product of the local posteriors, $q \propto \prod_i q_{\eta^i}$, whose natural parameter is $n\,\etaG_k + \sum_i \Delta\eta^i$. Iterated over synchronizations, the common prior and early evidence are counted geometrically many times, and in the Gaussian and Dirichlet families the precision or total pseudo-count is multiplied by at least $n$ at every synchronization whatever the data.

\begin{proposition}[Exact fusion]\label{thm:fusion}
Under Assumptions \ref{ass:ci}--\ref{ass:conj}, the centralized
posterior $P(\bth \mid \cD^{1:n}_{1:k+T})$ equals $q_{\etaG_{k+T}}$
with
\begin{equation}\label{eq:fusion}
\etaG_{k+T} = \etaG_k + \sum_{i=1}^n \Delta\eta^i,
\quad\text{equivalently}\quad
q^G_{k+T} \propto \frac{\prod_i q_{\eta^i}}{[q_{\etaG_k}]^{n-1}}.
\end{equation}
\end{proposition}
\begin{proof}
By Assumption~\ref{ass:ci} the centralized posterior is proportional to $q_{\etaG_k}\prod_i P(\cD^i_{k+1:k+T} \mid \bth)$, and by Bayes' rule and Assumption~\ref{ass:conj}, $\ln P(\cD^i_{k+1:k+T} \mid \bth) = \langle \Delta\eta^i, u(\bth)\rangle + c_i = \ln q_{\eta^i}(\bth) - \ln q_{\etaG_k}(\bth) + c_i'$ with $c_i, c_i'$ independent of $\bth$. Collecting natural parameters gives \eqref{eq:fusion}, and induction over synchronizations keeps the common belief equal to the centralized posterior.
\end{proof}

The division form of \eqref{eq:fusion} is the channel filter \cite{grime94} and the Bayesian committee machine \cite{tresp00}; the additive form is the information-filter update \cite{grime94}, which needs no count of how many copies of the common belief the received posteriors contain and stays proper when increments are lost or delayed, whereas division with a stale common term can produce negative pseudo-counts or indefinite precisions.

\begin{proposition}[Properness]\label{prop:proper}
For every subset of the increments $\{\Delta\eta^i\}$, the parameter obtained by adding that subset to $\etaG_k$ lies in $\mathcal{N}$: it is the natural parameter of the posterior obtained from the proper prior $q_{\etaG_0}$ and the data behind the received increments.
\end{proposition}
\begin{proof}
By Proposition~\ref{thm:fusion} applied to the subset, the fused parameter is the posterior parameter for the pooled data, and a posterior obtained from a proper prior and a bounded likelihood is proper.
\end{proof}

\section{Coordinated Exploration by Anticipated Evidence}\label{sec:pro}
Write $\mathcal{I}^i(\pi^i) := I(\bth; o^i_{\pi^i})$ for the information gain of robot $i$'s plan under the shared belief and $\mathcal{I}^{\mathrm{joint}}(\boldsymbol{\pi}) := I(\bth; o^1_{\pi^1},\dots,o^n_{\pi^n})$ for the gain of the team.

\begin{definition}[Exploration redundancy]\label{def:red}
$R(\boldsymbol{\pi}) := \sum_i \mathcal{I}^i(\pi^i) -
\mathcal{I}^{\mathrm{joint}}(\boldsymbol{\pi})$.
\end{definition}

\begin{proposition}[Redundancy equals total correlation]\label{thm:tc}
Under Assumption~\ref{ass:ci}, $R(\boldsymbol{\pi}) = \sum_i H(o^i_{\pi^i}) - H(o^{1:n}_{\boldsymbol{\pi}}) = \mathrm{TC}(o^1_{\pi^1},\dots,o^n_{\pi^n}) \ge 0$, with equality iff the planned streams are mutually independent under the predictive belief. For a given shared belief, $R$ is determined by the plans alone; the fusion rule, which acts after the plans have been executed, does not enter.
\end{proposition}
\begin{proof}
Write each mutual information as $H(o) - H(o \mid \bth)$. Given $\bth$ and fixed policies the streams are independent, so the conditional entropies cancel, and the remainder is the total correlation.
\end{proof}

\begin{definition}[Coordinated expected free energy]\label{def:defe}
Fix an order in which robots commit to plans. Robot $i$ replaces $\mathcal{I}^i(\pi^i)$ in \eqref{eq:efe} by the conditional gain $\tilde{\mathcal{I}}^i(\pi^i \mid \pi^{1:i-1}) := I(\bth;\, o^i_{\pi^i} \mid o^{1:i-1}_{\pi^{1:i-1}})$, so that $\tilde G^i(\pi^i) = G^i(\pi^i) + \mathcal{I}^i(\pi^i) - \tilde{\mathcal{I}}^i(\pi^i \mid \pi^{1:i-1})$.
\end{definition}

\begin{proposition}[Chain rule and greedy guarantee]\label{prop:chain}
(a) $\sum_i \tilde{\mathcal{I}}^i = \mathcal{I}^{\mathrm{joint}}$, so the corrected gains have zero redundancy. (b) Let the ground set consist of all (robot, candidate policy) pairs (streams of one robot under different policies are treated as conditionally independent given $\bth$, which affects only infeasible sets). Then $\mathcal{I}^{\mathrm{joint}}$ is monotone submodular \cite{krause05}, one policy per robot is a partition-matroid constraint, and sequential commitment in any fixed order is the locally greedy heuristic of \cite{fisher78}, which with exact conditional gains attains at least half of the maximum joint gain over all assignments of one candidate policy per robot \cite{fisher78,singh07}. The bound extends to the full objective whenever the task-value term is a nonnegative monotone submodular function of the committed plans, the own-state term of \eqref{eq:efe} being modular.
\end{proposition}
\begin{proof}
(a) is the chain rule. (b) Submodularity follows from conditional independence given $\bth$ \cite{krause05}, the guarantee is that of \cite{fisher78}, and a sum of monotone submodular functions is monotone submodular.
\end{proof}
Exact conditional gains require the teammates' predictive models, which trajectory sharing supplies \cite{singh07,atanasov15}; under Assumption~\ref{ass:conj} the effect of a teammate's future evidence on robot $i$'s belief is determined by the law of its evidence increment, which robot $j$ can compute under the shared belief and broadcast.

\begin{definition}[Anticipated increment]\label{def:anticipated}
After committing to $\pi^j$, robot $j$ broadcasts the mean of its evidence increment under the shared predictive belief,
\begin{equation}\label{eq:anticipated}
\widehat{\Delta\eta}^{j} := \E_{Q(o^j, s^j \mid \pi^j,\, \etaG)}
\big[\Delta\eta\big(o^j_{\pi^j}, s^j_{\pi^j}\big)\big],
\end{equation}
together with the expected change of any shared resources its plan consumes (Remark~\ref{rem:channels}). Robot $i$ evaluates its conditional gain under the anticipated belief
\begin{equation}\label{eq:counterfactual}
\tilde\eta^{<i} := \etaG + \sum_{j<i} \widehat{\Delta\eta}^{j},
\qquad
\tilde{\mathcal{I}}^i(\pi^i \mid \pi^{1:i-1}) \approx
I_{q_{\tilde\eta^{<i}}}\big(\bth;\, o^i_{\pi^i}\big).
\end{equation}
\end{definition}
The exact conditional gain is $\E[f(\etaG + \Delta)]$, where $\Delta$ is the random increment of the committed plans and $f(\eta) := I_{q_\eta}(\bth; o^i_{\pi^i})$, whereas \eqref{eq:counterfactual} evaluates $f$ at $\etaG + \E[\Delta]$.

\begin{lemma}[Expected gain of one categorical observation]\label{lem:mn}
Let $q(\bth_{\cdot j}) = \mathrm{Dir}(\mathbf{a}_{\cdot j})$ with $M_j$ outcomes and total count $N_j = \sum_m a_{mj}$, where context $j$ indexes one column of a likelihood or transition array.
\begin{enumerate}
\item[(i)] The exact expected gain of one observation in context $j$ is $\mathcal{I}_j := H(\mathbf{a}_{\cdot j}/N_j) - \E_q[H(\bth_{\cdot j})]$ with $\E_q[H(\bth_{\cdot j})] = \psi(N_j{+}1) - \sum_m \frac{a_{mj}}{N_j}\psi(a_{mj}{+}1)$, where $\psi$ is the digamma function. It equals $\sum_m (a_{mj}/N_j)\, W_{mj}$ with the outcome-specific gain
\begin{equation}\label{eq:Wexact}
\begin{split}
W_{mj} &= \KL\big(\mathrm{Dir}(\mathbf{a}_{\cdot j} + e_m) \,\|\,
\mathrm{Dir}(\mathbf{a}_{\cdot j})\big)\\
&= \ln\tfrac{N_j}{a_{mj}} + \psi(a_{mj}{+}1) - \psi(N_j{+}1)\\
&\approx \tfrac12\big(a_{mj}^{-1} - N_j^{-1}\big),
\end{split}
\end{equation}
where the approximation replaces $\psi(x{+}1)$ by $\ln x + \tfrac{1}{2x}$ and is the novelty weight of \cite{friston17,schwartenbeck19}. Under it the expected gain equals $(M_j-1)/(2N_j)$, whatever the distribution of the pseudo-counts among the outcomes.
\item[(ii)] The novelty approximation overestimates the exact gain, and
\begin{equation}\label{eq:novbound}
\frac{M_j-1}{2N_j} - \frac{1}{12N_j}\sum_m a_{mj}^{-1}
\;\le\; \mathcal{I}_j \;\le\; \frac{M_j-1}{2N_j}.
\end{equation}
\end{enumerate}
\end{lemma}
\begin{proof}
(i) The exact form is $H(o) - H(o \mid \bth)$ with $\E_{\mathrm{Dir}(\mathbf{a})}[-\sum_m \bth_m \ln \bth_m] = \psi(N{+}1) - \sum_m \frac{a_m}{N}\psi(a_m{+}1)$, and the Kullback--Leibler divergence between Dirichlet distributions gives \eqref{eq:Wexact}; with the second-order weights, $\sum_m \frac{a_{mj}}{N_j}\cdot\frac12\big(\frac{1}{a_{mj}} - \frac{1}{N_j}\big) = \frac{M_j - 1}{2N_j}$. (ii) Write $\psi(x{+}1) = \ln x + \tfrac{1}{2x} - r(x)$, where $0 \le r(x) \le \tfrac{1}{12x^2}$ and $r$ is decreasing on $x>0$; both facts follow from the bounds $\ln x - \tfrac{1}{2x} - \tfrac{1}{12x^2} < \psi(x) < \ln x - \tfrac{1}{2x}$ and $\psi'(x) > \tfrac{1}{x} + \tfrac{1}{2x^2}$ \cite{alzer97}. Substituting into $\E_q[H(\bth_{\cdot j})]$ and using $H(\mathbf{a}_{\cdot j}/N_j) = \ln N_j - \sum_m \frac{a_{mj}}{N_j}\ln a_{mj}$ and $\sum_m \frac{a_{mj}}{N_j}\cdot\frac{1}{2a_{mj}} = \frac{M_j}{2N_j}$ gives
\[
\mathcal{I}_j = \frac{M_j-1}{2N_j} + r(N_j) - \sum_m \frac{a_{mj}}{N_j}\, r(a_{mj}).
\]
Since $a_{mj} \le N_j$ and $r$ is decreasing, $\sum_m \frac{a_{mj}}{N_j} r(a_{mj}) \ge r(N_j)$, which gives the upper bound; dropping $r(N_j) \ge 0$ and using $r(a_{mj}) \le (12 a_{mj}^2)^{-1}$ gives the lower bound.
\end{proof}

\begin{theorem}[Sufficiency of the expected increment]\label{thm:exact}
Let $\Delta$ and $f$ be as above. If the natural parameter splits as $\eta = (\chi, \nu)$ such that the component $\Delta\nu$ of the increment is a deterministic function of the committed plans and $f$ depends on $\eta$ only through $\nu$, then $\E[f(\etaG + \Delta)] = f(\etaG + \E[\Delta])$ and the expected increment is a sufficient message.
\begin{enumerate}
\item[(a)] \emph{Gaussian information form.} With $\eta = (\Lambda\mu, \Lambda)$ the condition holds with $\nu = \Lambda$ whenever the committed policies fix their sampling locations, because the precision increment does not depend on the observed values and Gaussian mutual information is a function of precisions alone.
\item[(b)] \emph{Dirichlet counts.} Under the novelty weights of Lemma~\ref{lem:mn} the condition holds with $\nu = (N_j)_j$ whenever the planned visit counts $n_j$ are deterministic, because every observation raises $N_j$ by one whatever its outcome. Then \eqref{eq:counterfactual} equals the conditional gain of the novelty objective that discrete active inference optimizes, and differs from the exact conditional mutual information by at most $\frac{1}{12(N_j+n_j)}\sum_m a_{mj}^{-1}$ per planned observation in context $j$, uniformly over the teammates' outcomes. If $n_j$ is random with mean $\bar n_j$ and variance $v_j$, then
\begin{equation}\label{eq:cebound}
0 \;\le\; \E\Big[\tfrac{M_j-1}{2(N_j+n_j)}\Big] -
\tfrac{M_j-1}{2(N_j+\bar n_j)} \;\le\;
\tfrac{(M_j-1)\,v_j}{2N_j^{3}},
\end{equation}
so the relative error of \eqref{eq:counterfactual} is $O(v_j/N_j^2)$.
\item[(c)] \emph{Finite hypothesis classes.} If $\Theta$ is finite, $\eta$ is the vector of log posterior probabilities and $\Delta$ the vector of log-likelihoods of the planned readings, every component of which depends on the outcomes, so no split of the required form exists. The expected increment is then not sufficient, and \eqref{eq:counterfactual} can return the unconditional gain while the exact conditional gain is strictly smaller. The exact gain is the finite sum $\sum_{\delta} P(\Delta = \delta)\, f(\etaG + \delta)$ over the support of $\Delta$; for $m$ planned readings of a binary unknown with a common accuracy the support has $m+1$ points.
\end{enumerate}
\end{theorem}
\begin{proof}
For the general condition, write $\etaG = (\chi^G, \nu^G)$ and $\Delta = (\Delta\chi, \Delta\nu)$; since $f$ does not depend on $\chi$ and $\Delta\nu$ is deterministic, $f(\etaG + \Delta) = f(\chi^G + \E[\Delta\chi],\, \nu^G + \Delta\nu) = f(\etaG + \E[\Delta])$ almost surely, and taking expectations gives the claim. (a) The posterior precision after sampling at $x_{1:T}$ is $\Lambda + \sigma^{-2}\sum_\tau \phi(x_\tau)\phi(x_\tau)^\top$ whatever the values $y_{1:T}$, and the mutual information $\tfrac12 \ln\det(I + \sigma^{-2}\Phi_i \Lambda^{-1}\Phi_i^\top)$ of robot $i$'s planned samples $\Phi_i$ depends on $\eta$ only through $\Lambda$. (b) By Lemma~\ref{lem:mn}(i) the novelty gain in context $j$ depends on the counts only through $N_j$, and each planned observation raises $N_j$ by one, so $\nu = (N_j)_j$ satisfies the condition when the visit counts are deterministic. For the error bound, every realization $\mathbf{c}_{\cdot j}$ of the teammates' tallies has total $n_j$ and $a_{mj} + c_{mj} \ge a_{mj}$, so Lemma~\ref{lem:mn}(ii) at $\mathbf{a}_{\cdot j} + \mathbf{c}_{\cdot j}$ bounds the exact gain of one further observation by $\tfrac{M_j-1}{2(N_j+n_j)}$ from above and by that value minus $\tfrac{1}{12(N_j+n_j)}\sum_m a_{mj}^{-1}$ from below, uniformly in $\mathbf{c}_{\cdot j}$; averaging preserves both bounds, and \eqref{eq:counterfactual} returns the upper one. For \eqref{eq:cebound}, $g(n) := (M_j-1)/(2(N_j+n))$ is convex and decreasing on $n \ge 0$, so Jensen's inequality gives the lower bound, and a second-order expansion with $\sup_{n\ge 0} g''(n) = (M_j-1)/N_j^{3}$ gives the upper bound. (c) For a binary unknown with belief $p = 1/2$, one planned reading of accuracy $\zeta$ has predictive probability $q = p\zeta + (1-p)(1-\zeta) = 1/2$ of the outcome favoring $\bth = 1$ and expected log-odds increment $(2q-1)\ln\frac{\zeta}{1-\zeta} = 0$, although the reading moves the belief to $\zeta$ or $1-\zeta$. The exact conditional gain of a second reading is $H_b(\zeta^2 + (1-\zeta)^2) - H_b(\zeta)$, whereas \eqref{eq:counterfactual} returns $1 - H_b(\zeta)$ ($0.21$ against $0.53$ bits at $\zeta = 0.9$). The finite sum is the definition of conditional mutual information, and with a common accuracy the increment depends on the readings only through the number of favorable outcomes.
\end{proof}

In both families of Theorem~\ref{thm:exact}(a) and (b) the anticipated increment leaves the mean of the shared belief unchanged and raises its concentration where the teammate will observe, by the linearity of posterior expectations in conjugate families \cite{diaconis79}: the anticipated Gaussian belief has mean $\mu$ and precision $\Lambda + \sigma^{-2}\sum_\tau \phi_\tau\phi_\tau^\top$, and the anticipated Dirichlet belief has predictive $\mathbf{a}_{\cdot j}/N_j$ and total count $N_j + \bar n_j$. A finite hypothesis class has no concentration coordinate, so there the anticipated increment can vanish although the planned readings are informative, and the receiver instead averages over the finite support of $\Delta$, computed from the number and accuracy of the planned readings that the sender reports for each unknown.

\begin{corollary}[Per-round guarantee of Algorithm~1]\label{cor:round}
Let the conditional gains be evaluated exactly, by \eqref{eq:counterfactual} in cases (a) and (b) of Theorem~\ref{thm:exact} with deterministic visit counts and by the finite sum of case (c) otherwise, and let the round objective be the joint gain under the exact centralized posterior, the mutual information in cases (a) and (c) and the novelty objective \eqref{eq:dirteam} in case (b). Then in every round of Algorithm~1 in which the robots plan by that gain alone, as in a monitoring task, the committed plans attain at least half of its maximum over all assignments of one candidate policy per robot; the same holds for the full coordinated objective whenever its task-value term is a nonnegative monotone submodular function of the committed plans.
\end{corollary}
\begin{proof}
Proposition~\ref{thm:fusion} makes the shared belief the centralized posterior at the start of each round, Proposition~\ref{prop:chain}(b) gives the guarantee for the exact mutual information, and \eqref{eq:dirteam} gives it for the novelty objective, whose marginal gains are the corrected gains of the Dirichlet instantiation.
\end{proof}

\begin{remark}[Consumable resources]\label{rem:channels}
When plans consume shared resources, such as a rock that can be sampled once or an item collected by the first robot to arrive, the rollout of robot $i$ should also predict the environment as the committed teammates will leave it. Consumption enters \eqref{eq:efe} through the task term only and leaves Assumption~\ref{ass:ci} unaffected, so we let the shared summary be $z = (\eta, \rho)$, with $\rho$ the consumable resources, and let the anticipated message carry the expected change of both: anticipated evidence lowers the epistemic term where teammates will observe, anticipated consumption lowers the pragmatic term where teammates will collect, and Sec.~\ref{sec:experiments} measures the two effects separately.
\end{remark}

\section{Three Conjugate Families}\label{sec:inst}
\emph{Dirichlet counts.} For categorical models such as the likelihood and transition arrays of a discrete active inference model, with $q(\bth_{\cdot j}) = \mathrm{Dir}(\mathbf{a}_{\cdot j})$ per column, the increment is the tally $\sum_\tau o_\tau \otimes s_\tau$ and fusion reads $\mathbf{a}^G \leftarrow \mathbf{a}^G + \sum_i \Delta\mathbf{a}^i$. Tallies are nonnegative, and the total pseudo-count equals the prior count plus the number of transitions the team has experienced, a conservation that naive fusion violates by counting the prior $n$ times. By Lemma~\ref{lem:mn} the novelty gain of one observation in context $j$ is $(M_j-1)/(2N_j)$, which depends on a teammate's anticipated tallies only through their column sums, the expected numbers of visits to each context. A plan that visits context $j$ $m$ times, with the counts updated along the rollout, has gain $\sum_{k=0}^{m-1} (M_j-1)/(2(N_j+k)) = \tfrac12 (M_j-1)\,[\psi(N_j+m) - \psi(N_j)]$ by $\psi(x{+}1) - \psi(x) = 1/x$, and after $n_j$ anticipated teammate visits its corrected gain is $\tfrac12 (M_j-1)\,[\psi(N_j+n_j+m) - \psi(N_j+n_j)]$. Summing the corrected gains along the commitment order telescopes, so the team objective under the novelty weights is
\begin{equation}\label{eq:dirteam}
\tilde{\mathcal{I}}^{\mathrm{joint}}(\boldsymbol{\pi}) =
\tfrac12 \sum_j (M_j-1)\,\big[\psi\big(N_j + n_j(\boldsymbol{\pi})\big)
- \psi(N_j)\big],
\end{equation}
where $n_j(\boldsymbol{\pi})$ is the total number of planned visits to context $j$. Each $n_j$ is a modular function of the committed plans and $\psi$ is increasing and concave, so \eqref{eq:dirteam} is monotone submodular, the corrected gain of robot $i$ is its marginal gain, and Proposition~\ref{prop:chain}(b) applies to the novelty objective with the same $1/2$ guarantee. The anticipated message of robot $i$ thus reduces to its planned visit counts $(\bar n^i_j)_j$, from which every receiver reconstructs $\widehat{\Delta\mathbf{a}}^{i}_{\cdot j} = \bar n^i_j\, \mathbf{a}^G_{\cdot j}/N_j$. By \eqref{eq:novbound} the objective \eqref{eq:dirteam} and the corrected gains built from it are within $O\big(\sum_m a_{mj}^{-1}/N_j\big)$ per planned observation of the exact mutual information, a term that vanishes faster than the gain itself as evidence accumulates.

\emph{Gaussian information form.} Let the unknown be a function $f(x) = w^\top\phi(x)$ observed through $y_\tau = f(x_\tau) + \varepsilon_\tau$ with $\varepsilon_\tau \sim \cN(0,\sigma^2)$ and conjugate prior $w \sim \cN(\mu_0,\Lambda_0^{-1})$. The natural parameter is $(\Lambda\mu, \Lambda)$, the increment is $(\sum_\tau \phi_\tau y_\tau, \sum_\tau \phi_\tau\phi_\tau^\top)/\sigma^2$. The gain of sampling at $x$ is $\tfrac12\ln(1 + \sigma_f^2(x)/\sigma^2)$ with $\sigma_f^2(x) = \phi(x)^\top\Lambda^{-1}\phi(x)$. The anticipated increment is the predicted information matrix of decentralized sensor control \cite{grocholsky03,atanasov15} and the hallucinated observation of batch Bayesian optimization \cite{desautels14,ginsbourger10}; Theorem~\ref{thm:exact}(a) states that it is exact for deterministic sampling paths, and submodularity of the joint gain is classical \cite{krause05,singh07}.

\emph{Finite hypothesis classes.} When $\Theta = \{\bth_1,\dots, \bth_L\}$, as in \cite{wang25}, the natural parameter is the vector of log posterior probabilities, the increment the vector of log-likelihoods, and fusion \eqref{eq:fusion} adds log-likelihood vectors (for $L = 2$ the log-odds recursion of Sec.~\ref{sec:exp-rs}); by Theorem~\ref{thm:exact}(c) the receiver evaluates the finite sum instead of the mean increment, at $m+1$ terms per unknown for $m$ planned readings.

\section{Algorithm}\label{sec:algo}
Algorithm~1 assembles the two corrections into one synchronization loop, which we call \textsc{FedAIF-AE} (federated active inference with anticipated evidence).

\begin{figure}[t]
\centering
\fbox{\begin{minipage}{0.94\columnwidth}\small
\textbf{Algorithm 1: \textsc{FedAIF-AE}} (one synchronization round)
\smallskip\hrule\smallskip
\textbf{Given:} shared parameter $\etaG$, robot order $1,\dots,n$,
candidate policy sets $\Pi^i$, horizon $T$.
\begin{enumerate}\itemsep1pt
\item \textbf{Plan} (sequential commitment). For $i = 1,\dots,n$:
  form $\tilde\eta^{<i} = \etaG + \sum_{j<i}\widehat{\Delta\eta}^{j}$;
  for each $\pi \in \Pi^i$ roll out under $q_{\tilde\eta^{<i}}$ and
  evaluate the coordinated EFE (Def.~\ref{def:defe} via
  \eqref{eq:counterfactual}, or via the finite sum of
  Thm.~\ref{thm:exact}(c) for finite hypothesis classes); commit
  $\pi^i$ minimizing it; broadcast the anticipated message ($\widehat{\Delta\eta}^{i}$ of
  \eqref{eq:anticipated}; planned visit counts in the Dirichlet family;
  number and accuracy of planned readings for finite classes).
\item \textbf{Act} ($T$ steps, decentralized). Execute $\pi^i$; accumulate the realized
  increment $\Delta\eta^i \mathrel{+}= \Delta\eta(o^i_t, s^i_t)$.
\item \textbf{Fuse.} Broadcast $\Delta\eta^i$; all update
  $\etaG \leftarrow \etaG + \sum_i \Delta\eta^i$
  (Prop.~\ref{thm:fusion}); reset $\Delta\eta^i \leftarrow 0$.
\end{enumerate}
\end{minipage}}
\end{figure}

\emph{Complexity and communication.} Per round each robot broadcasts two sparse vectors of dimension $\dim\eta$ (plus $\dim\rho$ under Remark~\ref{rem:channels}), one anticipated and one realized. No observations, states, trajectories, or policies are transmitted, and a receiver applies an increment without the sender's sensor or motion model, since it is expressed in the coordinates of the shared belief. In the experiments of Sec.~\ref{sec:experiments} (2-byte index and 4-byte value per sparse entry) the anticipated and realized messages are $7.0$ and $4.2$ bytes in RockSample, $2$--$6$ and $1$--$3$ bytes in foraging, and $18.3$ and $18.9$\,kB in monitoring, where the sample locations exchanged by \cite{singh07} ($11.5$ bytes) and even the raw observations ($17.2$ bytes) are far more compact: with a large feature basis the increment trades bandwidth for independence from the sender's model. Fusion costs $O(n\dim\eta)$; planning adds one rollout per candidate policy, as in standard active inference, plus $O(m)$ terms per unknown for the finite sum of Theorem~\ref{thm:exact}(c), and sequential commitment requires $n$ short message rounds ($0.85$, $0.07$--$0.78$, and $13.7$\,ms per round in the three testbeds). With $n = 1$ both corrections vanish.

\section{Experiments}\label{sec:experiments}
We evaluate on three testbeds, one per case of Theorem~\ref{thm:exact}: a cooperative extension of RockSample \cite{smith04}, whose rock qualities form a finite hypothesis class (case c); a foraging task in the setting of \cite{tom25}, in a Dirichlet variant with unknown site yields (case b) and in its binary variant compared against centralized joint planning; and multi-robot field monitoring with Gaussian beliefs (case a) on synthetic and real fields.

\emph{Design.} The central prediction is that exact fusion cannot reduce planning redundancy (Prop.~\ref{thm:tc}), whereas anticipated evidence can. Fusion and planning are therefore manipulated independently, fusion \{naive (\textsc{NF}), corrected (\textsc{CF})\} $\times$ planning \{naive (\textsc{NP}), coordinated (\textsc{CP})\}, bracketed by \textsc{Oracle} (true $\bth$ known, pragmatic planning only) and \textsc{Local} (no communication). Metrics are team reward; the $\KL$ divergence of the team belief from the centralized posterior; behavioral redundancy, the number of (round, target) pairs in which two or more robots gather evidence about the same unknown, a proxy for $R(\boldsymbol\pi)$; the fraction of the round-optimal team objective attained by sequential commitment (Cor.~\ref{cor:round}); and communication and planning time. All arms share environment seeds within a trial, so comparisons are paired; we report means with 95\% confidence intervals and paired Wilcoxon signed-rank tests, and code and logs will be released.

\subsection{Cooperative RockSample}\label{sec:exp-rs}

\begin{table}[t]
\centering\footnotesize
\setlength{\tabcolsep}{3pt}
\begin{tabular}{@{}lcccc@{}}
\toprule
arm & reward & co-sense & $\KL$ & exact.\ err \\
\midrule
\textsc{Local} & $28.01 \pm 0.75$ & $11.68$ & --- & --- \\
\textsc{NF-NP} & $31.44 \pm 0.78$ & $10.23$ & $0.869$ & --- \\
\textsc{CF-NP} & $31.44 \pm 0.79$ & $10.13$ & $0.000$ & $10^{-12}$ \\
\textsc{NF-CP-mean} & $34.18 \pm 0.79$ & $5.92$ & $0.981$ & --- \\
\textsc{CF-CP-mean} & $34.08 \pm 0.80$ & $6.08$ & $0.000$ & $10^{-12}$ \\
\textsc{NF-CP-exact} & $34.23 \pm 0.81$ & $5.71$ & $1.216$ & --- \\
\textsc{CF-CP-exact} & $\mathbf{34.24 \pm 0.81}$ & $\mathbf{5.79}$ &
$\mathbf{0.000}$ & $10^{-12}$ \\
\textsc{Oracle} & $43.65 \pm 0.71$ & $2.44$ & --- & --- \\
\bottomrule
\end{tabular}
\caption{RS$(7,8)$, $n{=}3$, 1{,}000 paired trials (mean $\pm$ 95\%
CI).}
\label{tab:rs}
\end{table}

\emph{Setup.} Layout-randomized RS$(7,8)$ and RS$(11,11)$: rock positions uniform per trial, qualities Bernoulli$(0.5)$, $n \in \{2,3,5\}$ rovers starting at the west edge, deterministic moves, sensing correct with probability $\eta(d) = \tfrac12 + \tfrac12\cdot 2^{-d/20}$, rewards $+10$/$-10$ per sample and $+10$ for exit, $\gamma = 0.95$, horizon 100, $T = 5$. Beliefs are log-odds updated by log-likelihood ratios. Candidate policies are, per unsampled rock, the shortest path with two en-route readings and a sample-if-favorable terminal, with pragmatic value computed exactly by outcome enumeration, plus exit; commitment order is randomized per round and the common prior is $p_r = 0.5$. \textsc{CP-exact} evaluates the conditional gain by the finite sum of Theorem~\ref{thm:exact}(c) over the readings that committed teammates plan on each rock; \textsc{CP-mean} uses the mean increment of \eqref{eq:anticipated}, which Theorem~\ref{thm:exact}(c) predicts to carry no epistemic information at the common prior. Both arms share the pragmatic channel of Remark~\ref{rem:channels}.

\emph{Results} (Table~\ref{tab:rs}). Corrected arms match the centralized posterior to $10^{-12}$ (Table~\ref{tab:rs}); naive fusion has $\KL = 0.87$--$1.22$. Fusion alone changes neither redundancy ($-0.10 \pm 0.29$, $p = 0.52$) nor reward ($-0.00 \pm 0.21$, $p = 0.89$), whereas coordination with the exact gain removes $4.35 \pm 0.58$ co-sensing events ($p < 10^{-70}$) and adds $2.80 \pm 0.32$ reward ($p \approx 10^{-54}$); the effect replicates on RS$(11,11)$ ($+2.91 \pm 0.49$) and grows with $n$ (gap $+8.2$ at $n{=}5$, where naive co-sensing rises from $7.9$ to $14.4$). The exact and mean arms are statistically indistinguishable ($+0.17 \pm 0.26$ reward), because the epistemic term of a plan is at most $\ln 2$ nats per reading while the pragmatic stakes are $\pm 10$.

With a well-specified prior, naive fusion biases the belief ($\KL$) but not the reward. Under a misspecified common prior (log-odds $-1.5$) the fusion correction becomes decisive: reward rises from $26.5$ to $30.9$ and belief accuracy from $0.72$ to $0.95$, because naive fusion re-counts the wrong prior at every synchronization. Disabling each channel in turn (400 trials) separates the two mechanisms. With the mean increment the epistemic channel alone leaves co-sensing at $10.8$, the value without coordination, as Theorem~\ref{thm:exact}(c) predicts at the common prior; with the exact finite sum it reduces co-sensing to $10.4$, the pragmatic channel alone to $7.0$, and both channels together to $6.2$ ($6.4$ with the mean increment). The epistemic channel is thus active but small in this consumable-reward domain; Sec.~\ref{sec:exp-mon} is the complementary case.

\subsection{Foraging: Dirichlet Yields and Centralized Joint Planning}\label{sec:exp-tom}
\emph{Setup.} $k \times k$ grids hold $n$ apple sites; the step cost is $0.3$, the horizon $3k$, and idling is allowed. In the \emph{binary variant} (the setting of \cite{tom25}) one site is known, the others hold an apple with probability $0.5$, the first arrival collects it ($+10$), and readings of accuracy $0.9$ are obtained within one cell. In the \emph{Dirichlet variant} sites are not consumed: each step the first robot to arrive harvests once and observes a yield in $\{0,1,2\}$ apples ($5$ per apple) drawn from the site's unknown yield distribution $\bth_s \sim \mathrm{Dir}(1,1,1)$, beliefs are Dirichlet counts with prior $(1,1,1)$ and one known good site, and one commitment is one planned harvest, so visit counts are deterministic and Theorem~\ref{thm:exact}(b) applies. Three planners share model and preferences and differ only in coordination: \emph{naive}; \emph{joint}, a centralized enumeration of all assignments of one candidate policy per robot at cost $O((|\Pi|{+}1)^n)$, which attains the round-optimal team objective and upper-bounds every decentralized planner over the same policy space, including the recursive planner of \cite{tom25}; and \textsc{CF-CP} at cost $O(n|\Pi|)$, with conditional gains by the finite sum of Theorem~\ref{thm:exact}(c) in the binary variant and by the mean count increment in the Dirichlet variant; 300 paired trials, 150 for $n \ge 4$.

\begin{figure*}[t]
\centering
\includegraphics[width=0.58\textwidth]{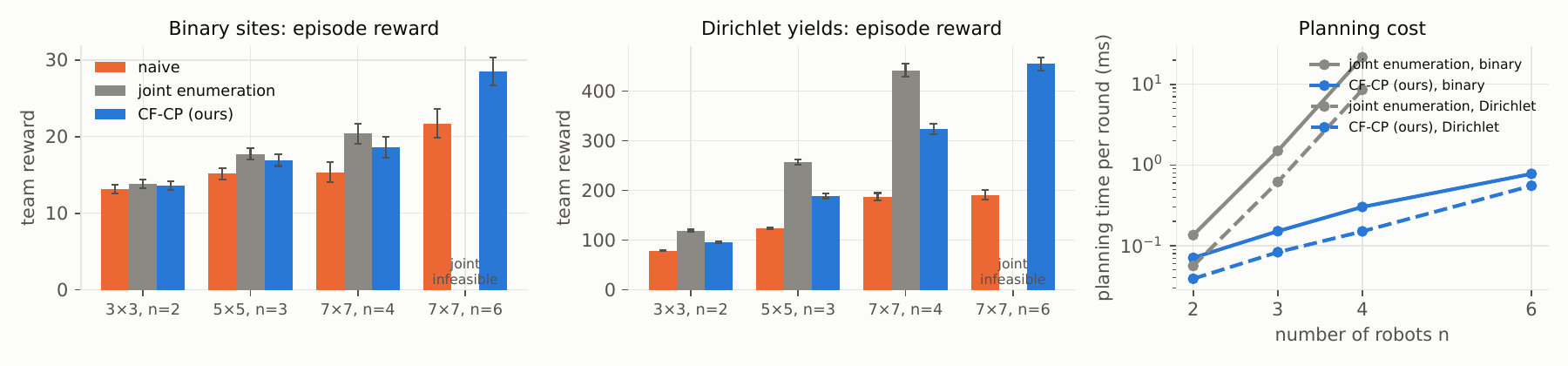}
\caption{Foraging, binary (left) and Dirichlet (middle) variants, and planning cost per round (right): \textsc{CF-CP} recovers half to two thirds of the gap between naive planning and centralized joint enumeration at linear cost, and remains usable at $n{=}6$, where enumeration is infeasible.}
\label{fig:tom}
\end{figure*}
\emph{Results} (Fig.~\ref{fig:tom}). In the binary variant at $3{\times}3$, $n{=}2$, \textsc{CF-CP} and joint enumeration are within noise ($13.61 \pm 0.56$ against $13.88 \pm 0.56$; naive $13.15$), and \textsc{CF-CP} co-targets in $0.34$ of rounds against $0.82$ for naive. At larger scale joint enumeration keeps a small advantage ($17.75$ against $16.93$ at $n{=}3$; $20.40$ against $18.62$ at $n{=}4$; naive $15.2$ and $15.4$) while the costs separate: enumeration grows tenfold per added robot ($0.14$, $1.5$, $22$\,ms per round) and is not run at $n{=}6$, whereas \textsc{CF-CP} grows linearly ($0.07$ to $0.78$\,ms) and at $n{=}6$ exceeds naive by $6.8$ ($28.5$ against $21.8$). Sequential commitment attains $87\%$, $83\%$, and $78\%$ of the round-optimal team objective at $n = 2, 3, 4$, and $91$--$98\%$ of the joint planner's reward, well above the $1/2$ bound of Prop.~\ref{prop:chain}(b). In the Dirichlet variant \textsc{CF-CP} earns $96$, $189$, and $323$ against $119$, $257$, and $442$ for joint enumeration and $79$, $123$, and $187$ for naive at $n = 2, 3, 4$, which is $81$--$85\%$ of the round-optimal objective; at $n{=}6$ it earns $454$ against $191$ for naive. The mean count increment reproduced the enumerated conditional gain to $10^{-16}$ on every commitment, as Theorem~\ref{thm:exact}(b) states.

\subsection{Multi-Robot Field Monitoring}\label{sec:exp-mon}

\emph{Setup.} $n{=}3$ waypoint robots monitor a scalar field on the unit square, sampling along their paths ($\sigma = 0.1$, one sample per $0.15$ of travel) under a travel budget of $4.0$. Ground truth fields are seeded draws from the model class (RBF mixtures, $m = 144$ features, lengthscale $0.12$). Waypoints are chosen from an $8{\times}8$ grid by novelty minus travel cost, fusion occurs every round, and RMSE is measured on a held-out $40{\times}40$ grid over 100 paired seeds. Coordination here acts through the epistemic term alone (Rem.~\ref{rem:channels}). Fusion baselines are the naive product (\textsc{NF}), covariance intersection with equal weights (\textsc{CI}, $\Lambda \leftarrow \Lambda^G + \tfrac1n\sum_i \Delta\Lambda^i$), the conservative rule for unknown common information, and exact additive fusion (\textsc{CF}). Planning baselines are independent greedy planning (\textsc{NP}), a Voronoi partition of the domain by robot positions with greedy planning restricted to each cell \cite{kemna17}, a static serpentine \textsc{Lawnmower}, and \textsc{CP}. By Theorem~\ref{thm:exact}(a), \textsc{CF-CP} selects the same plans as sequential greedy conditional mutual information with trajectory sharing \cite{singh07}, which we verified on every seed; the two differ only in the message.

\begin{figure*}[t]
\centering
\includegraphics[width=0.58\textwidth]{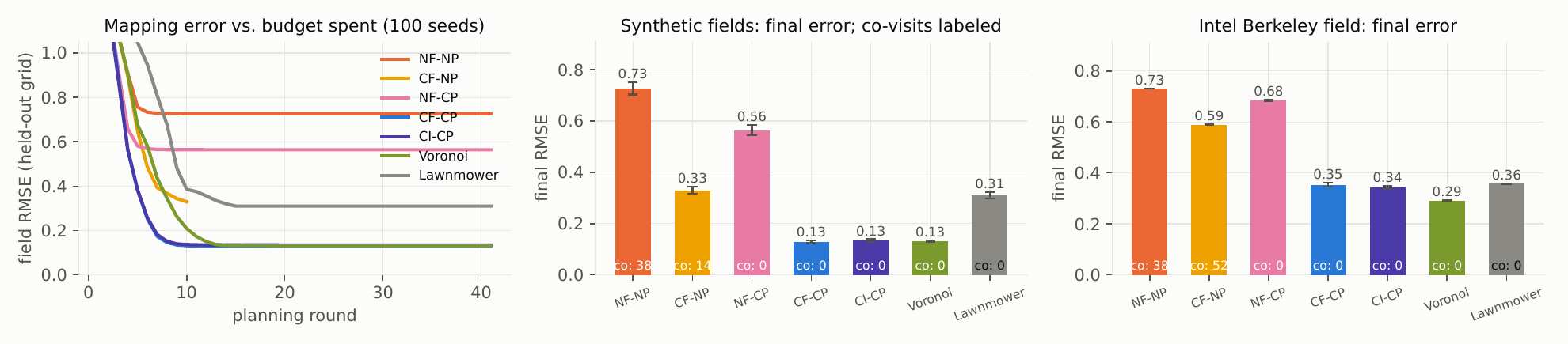}
\caption{Monitoring, 100 paired seeds: mapping error against budget on synthetic fields (left); final RMSE with co-visit counts on synthetic fields (middle) and on the Intel Berkeley temperature field (right). Both corrections are necessary; under the model mismatch of the real field the Voronoi partition overtakes coordinated planning.}
\label{fig:mon-quant}
\end{figure*}

\emph{Results} (Fig.~\ref{fig:mon-quant}). Coordination reduces co-committed waypoints from 14 per episode to zero and final RMSE from $0.330 \pm 0.014$ to $0.129 \pm 0.005$ (paired $-0.200 \pm 0.014$), a factor of $2.4$ below \textsc{Lawnmower} ($0.310 \pm 0.013$). The fusion correction matters more here than in Sec.~\ref{sec:exp-rs} even with a correct prior: naive fusion multiplies the prior precision by $n$ each round, predictive variance and information gain collapse, exploration stops, and RMSE remains at $0.727 \pm 0.023$; under \textsc{NF-CP} it is $0.564$ despite zero co-visits, because robots sent to different locations still plan under the same overconfident uncertainty. Covariance intersection reaches $0.134 \pm 0.005$ and the Voronoi partition $0.131 \pm 0.003$, both within noise of \textsc{CF-CP} (paired $0.005 \pm 0.006$ and $0.002 \pm 0.005$): on fields drawn from the model class, a conservative rule that inflates the variance by about $n$ while leaving the mean unchanged costs little mapping error, and a geometric partition of three robots on a square is nearly as good as information-driven coordination. Exact fusion remains the only rule that matches the centralized posterior. Theorem~\ref{thm:exact}(a) held exactly on every committed plan.

\subsubsection{Real field}\label{sec:exp-mon-real} The same protocol was run on the Intel Berkeley temperature field \cite{intel04}: the ten-minute window with the most reporting motes (53 readings), coordinates mapped to the unit square, temperatures standardized, ground truth by thin-plate-spline interpolation on the $40{\times}40$ grid, and lengthscale ($0.065$) and prior precision ($0.90$) fitted by marginal likelihood. Final RMSE over 100 paired seeds is $0.731$ (\textsc{NF-NP}), $0.589$ (\textsc{CF-NP}), $0.684$ (\textsc{NF-CP}), $0.354 \pm 0.007$ (\textsc{CF-CP}), $0.343 \pm 0.007$ (\textsc{CI-CP}), $0.292 \pm 0.002$ (Voronoi), and $0.357$ (\textsc{Lawnmower}). Both corrections remain necessary (\textsc{CF-CP} improves on \textsc{CF-NP} by $0.236 \pm 0.007$ and on \textsc{NF-CP} by $0.33$), but the ranking among the coordinated and partition-based planners changes: the fitted lengthscale is below the feature spacing, so the model is misspecified, and the Voronoi partition, which spreads samples without consulting the model, is now best, with covariance intersection and \textsc{Lawnmower} within noise of \textsc{CF-CP}.

\section{Conclusion}\label{sec:conclusion} The common belief of a robot team enters every local posterior at fusion and every robot's information gain at planning. Adding evidence increments to the shared natural parameter, realized at fusion and expected at planning, removes both errors, and the expected increment is an exact message whenever the gain depends on the increment only through a component that the plan fixes, which holds for Gaussian and Dirichlet beliefs and fails for finite hypothesis classes. Exact fusion alone leaves exploration redundancy unchanged; in field monitoring the mapping error of centralized planning is reached only with both corrections; sequential commitment recovers most of the value of joint planning at linear cost; and on the real temperature field, where the model is misspecified, a Voronoi partition performs better than information-driven coordination.

\emph{Limitations.} Conjugacy excludes learned belief representations, and conditional independence given $\bth$ excludes correlated sensor noise, which would require tracking the provenance of shared evidence \cite{dagan23}. Under partial communication each robot must record which increments it has already added, as the channel filter does for tree networks \cite{grime94}. Sequential commitment needs $n$ message rounds before the team acts; distributed sequential greedy assignment \cite{corah17} lets groups commit in parallel at the price of a weaker bound. The theory-of-mind baseline is a joint-enumeration surrogate for \cite{tom25}, and the real-field experiment uses one interpolated time slice.


\begin{thebibliography}{99}\footnotesize

\bibitem{wang25} Y.~Wang, P.~Wu, M.~Imani. Federated posterior sharing
for multi-agent systems in uncertain environments. \emph{L4DC}, PMLR
283, 2025.

\bibitem{dacosta20} L.~Da~Costa, T.~Parr, N.~Sajid, S.~Veselic,
V.~Neacsu, K.~Friston. Active inference on discrete state-spaces: a
synthesis. \emph{J.\ Math.\ Psych.}, 99:102447, 2020.

\bibitem{friston17} K.~Friston, M.~Lin, C.~Frith, G.~Pezzulo,
J.~Hobson, S.~Ondobaka. Active inference, curiosity and insight.
\emph{Neural Comput.}, 29(10):2633--2683, 2017.

\bibitem{schwartenbeck19} P.~Schwartenbeck, J.~Passecker, T.~Hauser,
T.~FitzGerald, M.~Kronbichler, K.~Friston. Computational mechanisms of
curiosity and goal-directed exploration. \emph{eLife}, 8:e41703, 2019.

\bibitem{friston24} K.~Friston, T.~Parr, C.~Heins, et al. Federated
inference and belief sharing. \emph{Neurosci.\ Biobehav.\ Rev.},
156:105500, 2024.

\bibitem{tresp00} V.~Tresp. A Bayesian committee machine. \emph{Neural
Comput.}, 12(11):2719--2741, 2000.

\bibitem{grime94} S.~Grime, H.~Durrant-Whyte. Data fusion in
decentralized sensor networks. \emph{Control Eng.\ Practice},
2(5):849--863, 1994.

\bibitem{krause05} A.~Krause, C.~Guestrin. Near-optimal nonmyopic value
of information in graphical models. \emph{UAI}, 2005.

\bibitem{singh07} A.~Singh, A.~Krause, C.~Guestrin, W.~Kaiser,
M.~Batalin. Efficient planning of informative paths for multiple
robots. \emph{IJCAI}, 2007.

\bibitem{fisher78} M.~Fisher, G.~Nemhauser, L.~Wolsey. An analysis of
approximations for maximizing submodular set functions---II.
\emph{Math.\ Prog.\ Studies}, 8:73--87, 1978.

\bibitem{smith04} T.~Smith, R.~Simmons. Heuristic search value
iteration for POMDPs. \emph{UAI}, 2004.

\bibitem{best19} G.~Best, O.~Cliff, T.~Patten, R.~Mettu, R.~Fitch.
Dec-MCTS: Decentralized planning for multi-robot active perception.
\emph{Int.\ J.\ Robot.\ Res.}, 38(2--3):316--337, 2019.

\bibitem{intel04} P.~Bodik, W.~Hong, C.~Guestrin, S.~Madden,
M.~Paskin, R.~Thibaux. Intel Berkeley Research Lab sensor data, 2004.

\bibitem{tom25} R.~J.~Pitliya, O.~\c{C}atal, T.~Van~de~Maele,
C.~Pezzato, T.~Verbelen. Theory of mind using active inference: a
framework for multi-agent cooperation. \emph{Int.\ Workshop on Active
Inference (IWAI 2025)}, Springer CCIS, 2026. arXiv:2508.00401.

\bibitem{wu24} P.~Wu, T.~Imbiriba, V.~Elvira, P.~Closas. Bayesian data
fusion with shared priors. \emph{IEEE Trans.\ Signal Process.},
72:275--288, 2024.

\bibitem{corah18} M.~Corah, N.~Michael. Distributed matroid-constrained
submodular maximization for multi-robot exploration. \emph{Auton.\
Robots}, 43:485--501, 2019.

\bibitem{desautels14} T.~Desautels, A.~Krause, J.~W.~Burdick.
Parallelizing exploration--exploitation tradeoffs in Gaussian process
bandit optimization. \emph{J.\ Mach.\ Learn.\ Res.}, 15:4053--4103,
2014.

\bibitem{ginsbourger10} D.~Ginsbourger, R.~Le~Riche, L.~Carraro.
Kriging is well-suited to parallelize optimization. In
\emph{Computational Intelligence in Expensive Optimization Problems},
Springer, 2010.

\bibitem{burgard05} W.~Burgard, M.~Moors, C.~Stachniss, F.~Schneider.
Coordinated multi-robot exploration. \emph{IEEE Trans.\ Robot.},
21(3):376--386, 2005.

\bibitem{dimakopoulou18} M.~Dimakopoulou, B.~Van~Roy. Coordinated
exploration in concurrent reinforcement learning. \emph{ICML}, 2018.

\bibitem{catal24} O.~\c{C}atal, T.~Van~de~Maele, R.~J.~Pitliya,
M.~Albarracin, C.~Pattisapu, T.~Verbelen. Belief sharing: a blessing or
a curse. \emph{Int.\ Workshop on Active Inference}, 2024.

\bibitem{dagan23} O.~Dagan, N.~R.~Ahmed. Exact and approximate
heterogeneous Bayesian decentralized data fusion. \emph{IEEE Trans.\
Robot.}, 39, 2023.

\bibitem{schlotfeldt18} B.~Schlotfeldt, D.~Thakur, N.~Atanasov,
V.~Kumar, G.~J.~Pappas. Anytime planning for decentralized multi-robot
active information gathering. \emph{IEEE Robot.\ Autom.\ Lett.},
3(2):1025--1032, 2018.

\bibitem{kantaros21} Y.~Kantaros, B.~Schlotfeldt, N.~Atanasov,
G.~J.~Pappas. Sampling-based planning for non-myopic multi-robot
information gathering. \emph{Auton.\ Robots}, 45, 2021.

\bibitem{kemna17} S.~Kemna, J.~G.~Rogers, C.~Nieto-Granda, S.~Young,
G.~S.~Sukhatme. Multi-robot coordination through dynamic Voronoi
partitioning for informative adaptive sampling in
communication-constrained environments. \emph{ICRA}, 2017.


\bibitem{grocholsky03} B.~Grocholsky, A.~Makarenko, H.~Durrant-Whyte.
Information-theoretic coordinated control of multiple sensor
platforms. \emph{ICRA}, pp.~1521--1526, 2003.

\bibitem{atanasov15} N.~Atanasov, J.~Le~Ny, K.~Daniilidis,
G.~J.~Pappas. Decentralized active information acquisition: Theory and
application to multi-robot SLAM. \emph{ICRA}, pp.~4775--4782, 2015.

\bibitem{wakayama25} S.~Wakayama, A.~Candela, P.~Hayne, N.~Ahmed.
Active inference for bandit-based autonomous robotic exploration with
dynamic preferences. \emph{IEEE Trans.\ Robot.}, 41:3841--3851, 2025.

\bibitem{diaconis79} P.~Diaconis, D.~Ylvisaker. Conjugate priors for
exponential families. \emph{Ann.\ Statist.}, 7(2):269--281, 1979.
\bibitem{ashman22} M.~Ashman, T.~D.~Bui, C.~V.~Nguyen, S.~Markou,
A.~Weller, S.~Swaroop, R.~E.~Turner. Partitioned variational inference:
a framework for probabilistic federated learning. arXiv:2202.12275,
2022.

\bibitem{corah17} M.~Corah, N.~Michael. Efficient online multi-robot
exploration via distributed sequential greedy assignment.
\emph{Robotics: Science and Systems}, 2017.

\bibitem{alzer97} H.~Alzer. On some inequalities for the gamma and psi
functions. \emph{Math.\ Comp.}, 66(217):373--389, 1997.

\end{thebibliography}
\end{document}